\documentclass[12pt]{article}
\usepackage{fullpage}

\title{Bandit Convex Optimization:\\ $\sqrt{T}$ Regret in One Dimension}

\author{%
S\'ebastien Bubeck\\
Microsoft Research\\
\email{sebubeck@microsoft.com}
\and
Ofer Dekel\\
Microsoft Research\\
\email{oferd@microsoft.com}
\and
Tomer Koren\thanks{Parts of this work were done while the author was at Microsoft Research, Redmond.}\\
Technion\\
\email{tomerk@technion.ac.il}
\and
Yuval Peres\\
Microsoft Research\\
\email{peres@microsoft.com}
}

\usepackage[colorlinks,linkcolor=blue,citecolor=blue]{hyperref}
\usepackage[round]{natbib}
\usepackage{bm}
\usepackage[ruled,noline]{algorithm2e}
\usepackage{enumitem}
\usepackage{framed}

\usepackage{tikz}
\usepackage{pgfplots}
\pgfplotsset{compat=newest}
\pgfplotsset{plot coordinates/math parser=false}
\usetikzlibrary{patterns,positioning}

\usepackage{amsmath,amsfonts,amssymb}
\usepackage{mathtools}
\usepackage{amsthm} % better to load after ams math

\usepackage[capitalise]{cleveref}
\usepackage{xcolor}
\usepackage{dsfont}

\newcommand{\ignore}[1]{}

\newcommand{\email}[1]{\texttt{#1}}

\theoremstyle{plain}
\newtheorem{theorem}{Theorem}
\newtheorem{lemma}[theorem]{Lemma}

\newtheorem*{theorem*}{Theorem}
\newtheorem*{lemma*}{Lemma}
\newtheorem*{corollary*}{Corollary}
\newtheorem*{proposition*}{Proposition}
\newtheorem*{claim*}{Claim}
\newtheorem*{fact*}{Fact}

\theoremstyle{definition}

\newtheorem*{definition*}{Definition}

\newtheorem*{remark*}{Remark}

\newtheorem*{example*}{Example}

\theoremstyle{plain}
\newtheorem*{theoremaux}{\theoremauxref}
\gdef\theoremauxref{1}

\newenvironment{repthm}[2][]{%
  \def\theoremauxref{\cref{#2}}
  \begin{theoremaux}[#1]
}{%
  \end{theoremaux}
}

\DeclareMathAlphabet{\mathbfsf}{\encodingdefault}{\sfdefault}{bx}{n}

\DeclareMathOperator*{\argmin}{arg\,min}

\let\Pr\relax
\DeclareMathOperator{\Pr}{\mathbb{P}}

\newcommand{\lr}[1]{\!\left(#1\right)\!}
\newcommand{\lrbig}[1]{\big(#1\big)}
\newcommand{\lrBig}[1]{\Big(#1\Big)}
\newcommand{\lrbra}[1]{\!\left[#1\right]\!}

\newcommand{\lrset}[1]{\left\{#1\right\}}

\newcommand{\set}[1]{\{#1\}}
\newcommand{\abs}[1]{|#1|}
\newcommand{\norm}[1]{\|#1\|}

\renewcommand{\t}[1]{\smash{\tilde{#1}}}
\newcommand{\wt}[1]{\smash{\widetilde{#1}}}

\renewcommand{\O}{O}

\newcommand{\tO}{\wt{\O}}
\newcommand{\tTheta}{\wt{\Theta}}
\newcommand{\E}{\mathbb{E}}
\newcommand{\EE}[1]{\E\lrbra{#1}}
\newcommand{\var}{\mathrm{Var}}

\newcommand{\ind}[1]{\mathds{1}\!\lrset{#1}}

\newcommand{\st}{\star}

\newcommand{\kl}[2]{\mathsf{D}_\mathrm{KL}(#1,#2)}

\newcommand{\ent}{\mathsf{H}}
\newcommand{\info}{\mathsf{I}}

\newcommand{\reals}{\mathbb{R}}
\newcommand{\eps}{\epsilon}

\newcommand{\del}{\delta}
\newcommand{\Del}{\Delta}

\newcommand{\half}{\frac{1}{2}}
\newcommand{\thalf}{\tfrac{1}{2}}

\newcommand{\eq}{~=~}
\renewcommand{\leq}{~\le~}
\renewcommand{\geq}{~\ge~}

\newcommand{\K}{\mathcal{K}}
\newcommand{\F}{\mathcal{F}}

\newcommand{\X}{\mathcal{X}}
\renewcommand{\H}{\mathcal{H}}

\newcommand{\R}{r}
\newcommand{\V}{v}

\begin{document} 
\maketitle

\begin{abstract}
We analyze the minimax regret of the adversarial bandit convex
optimization problem.  Focusing on the one-dimensional case, we prove
that the minimax regret is $\tTheta(\sqrt{T})$ and partially
resolve a decade-old open problem. Our analysis is
non-constructive, as we do not present a concrete algorithm that
attains this regret rate. Instead, we use minimax duality to reduce
the problem to a Bayesian setting, where the convex loss functions are
drawn from a worst-case distribution, and then we solve the Bayesian
version of the problem with a variant of Thompson Sampling. 
Our analysis features a novel use of convexity, formalized as a ``local-to-global'' property of convex functions, that may be of independent interest.%
%\footnote{This paper is eligible for a best student paper award.}
\end{abstract}

\section{Introduction}

Online convex optimization with bandit feedback, commonly known as
bandit convex optimization, can be described as a $T$-round game,
played by a randomized player in an adversarial environment.  Before
the game begins, the adversarial environment chooses an arbitrary
sequence of $T$ bounded convex functions $f_1,\ldots,f_T$, where each
$f_t:\K \mapsto [0,1]$ and $\K$ is a fixed convex and compact set in
$\reals^n$. On round $t$ of the game, the player chooses a point $X_t
\in \K$ and incurs a loss of $f_t(X_t)$. The player observes the value
of $f_t(X_t)$ and nothing else, and she uses this information to
improve his choices going forward.  The player's performance is
measured in terms of his $T$-round \emph{regret}, defined as
$\sum_{t=1}^T f_t(X_t) - \min_{x \in \K} \sum_{t=1}^T f_t(x)$. In
words, the regret compares the player's cumulative loss to that of the
best fixed point in hindsight.

While regret measures the performance of a specific player against a
specific loss sequence, the inherent difficulty of the game is
measured using the notion of \emph{minimax regret}. Informally, the
game's minimax regret is the regret of an optimal player when she faces
the worst-case loss sequence. Characterizing the minimax regret of
bandit convex optimization is one of the most elusive open problems in
the field of online learning. For general bounded convex loss functions, 
\cite{FKM05} presents an algorithm that guarantees
a regret of $\tO(T^{5/6})$---and this is the best known upper bound on
the minimax regret of the game. Better regret rates can be guaranteed
if additional assumptions are made: for Lipschitz functions the
regret is $\tO(T^{3/4})$ \citep{FKM05}, for Lipschitz and
strongly convex losses the regret is $\tO(T^{2/3})$ \citep{ADX10}, and 
for smooth functions the regret is $\tO(T^{2/3})$
\citep{saha2011improved}. In all of the aforementioned settings, the
best known lower bound on minimax regret is $\Omega(\sqrt{T})$
\citep{DHK08}, and the challenge is to bridge the gap between the
upper and lower bounds. In a few special cases, the gap is resolved
and we know that the minimax regret is exactly $\tTheta(\sqrt{T})$;
specifically, when the loss functions are both smooth and
strongly-convex \citep{HL14}, when they are Lipschitz and linear
\citep{DHK08,AHR08}, or when they are Lipschitz and drawn i.i.d.~from 
a fixed and unknown distribution \citep{AFHKR11}.

In this paper, we resolve the open problem in the one-dimensional
case, where $\K = [0,1]$, by proving that the minimax regret with
arbitrary bounded convex loss functions is $\tTheta(\sqrt{T})$.
Formally, we prove the following theorem.
\begin{theorem}[main result] %\label{thm:main}
There exists a randomized player strategy that
relies on bandit feedback and guarantees an expected regret of 
$\O(\sqrt{T} \log{T} )$ against any sequence of convex loss functions
$f_1,\ldots,f_T : [0,1] \mapsto [0,1]$.
\end{theorem}
The one-dimensional case has received very little special attention,
and the best published result is the $\tO(T^{5/6})$ bound mentioned
above, which holds in any dimension. However, by discretizing the
domain $[0,1]$ appropriately and applying a standard multi-armed
bandit algorithm, one can prove a tighter bound of $\tO(T^{2/3})$; see
\cref{sec:constructive} for details. It is worth noting that
replacing the convexity assumption with a Lipschitz assumption also
gives an upper bound of $\tO(T^{2/3})$ \citep{Kle04}. However,
obtaining the tight upper bound of $\tO(\sqrt{T})$ requires a more
delicate analysis, which is the main focus of this paper.

Our tight upper bound is non-constructive, in the sense that we do not
describe an algorithm that guarantees a $\tO(\sqrt{T})$ regret for any
loss sequence. Instead, we use minimax duality to reduce the problem of bounding the
adversarial minimax regret to the problem of upper bounding the
analogous \emph{maximin regret} in a \emph{Bayesian} setting.  Unlike
our original setting, where the sequence of convex loss functions is
chosen adversarially, the loss functions in the Bayesian setting are
drawn from a probability distribution, called the \emph{prior}, 
which is known to the player. 
The idea of using minimax duality to study minimax regret is not new 
\citep[see, e.g.,][]{AABR09, GPS14}; however, to the best of our knowledge, we
are the first to apply this technique to prove upper bounds in a
bandit feedback scenario.

After reducing our original problem to the Bayesian setting, we design
a novel algorithm for Bayesian bandit convex optimization (in one
dimension) that guarantees $\tO(\sqrt{T})$ regret for any prior
distribution.  Since our main result is non-constructive to begin
with, we are not at all concerned with the computational efficiency of
this algorithm. We first discretize the domain $[0,1]$ and treat each
discrete point as an arm in a multi-armed bandit problem. We then
apply a variant of the classic Thompson Sampling strategy
\citep{Tho33} that is designed to exploit the fact that the loss
functions are all convex. We adapt the analysis of Thompson Sampling
in \cite{russo2014information} to our algorithm and extend it to
arbitrary joint prior distributions over sequences of loss functions
(not necessarily i.i.d.~sequences).

The significance of the convexity assumption is that it enables us to
obtain regret bounds that scale \emph{logarithmically} with the number
of arms, which turns out to be the key property that leads to the
desired $\tO(\sqrt{T})$ upper bound. Intuitively, convexity ensures
that a change to the loss value of one arm influences the loss values
in many of the adjacent arms. Therefore, even the worst case prior
distribution cannot hide a small loss in one arm without globally
influencing the loss of many other arms.  Technically, this aspect of
our analysis boils down to a basic question about convex functions:
given two convex functions $f:\K \mapsto[0,1]$ and $g:\K \mapsto[0,1]$
such that $f(x) < \min_y g(y)$ at some point $x\in\K$, how small can
$\|f-g\|$ be (where $\|\cdot\|$ is an appropriate norm over the
function space)?  In other words, if two convex functions differ
locally, how similar can they be globally? We give an answer to this question in
the one-dimensional case.

The paper is organized as follows. 
We begin in \cref{sec:bayes} where
we define the setting of Bayesian online optimization, establish basic
techniques for the analysis of Bayesian online algorithms, and
demonstrate how to readily recover some of the known minimax regret
bounds for the full information case by bounding the Bayesian
regret. 
Then, in \cref{sec:loc-to-glob}, we
prove the key structural lemma by which we exploit the convexity of the loss
functions. 
\cref{sec:algo} is the main part of the paper, where we give
our algorithm for Bayesian bandit convex optimization (in one
dimension) and analyze its regret. 
We conclude the paper in \cref{sec:conc} with a few remarks and open problems.

\section{From Adversarial to Bayesian Regret} 
\label{sec:bayes}

In this section, we show how regret bounds for an adversarial online
optimization setting can be obtained via a Bayesian analysis.  Before
explaining this technique in detail, we first formalize two variants of the online optimization problem: the adversarial setting and the Bayesian setting.
%. the settings of adversarial online optimization, and Bayesian online optimization.

We begin with the standard, adversarial online optimization setup.
As described above, in this setting the player plays a $T$-round game, during which he
chooses a sequence of points $X_{1:T}$,%
\footnote{Throughout the paper, we use the notation $a_{s:t}$ as shorthand for
the sequence $a_s,\ldots,a_t$.}
where $X_t \in \K$ for all $t$. The player's randomized policy for
choosing $X_{1:T}$ is defined by a sequence of deterministic functions
$\rho_{1:T}$, where each $\rho_t : [0,1]^{t-1} \mapsto \Del(\K)$ (here
$\Del(\K)$ is the set of probability distributions over $\K$). On
round $t$, the player uses $\rho_t$ and his past observations to define
the probability distribution
$$
 	\pi_t \eq \rho_t \big(f_1(X_1),\ldots,f_{t-1}(X_{t-1})\big)~,
$$ 
and then draws a concrete point $X_t \sim \pi_t$. 
Even though $\rho_t$ is a
deterministic function, the probability distribution $\pi_t$ is itself
a random variable, because it depends on the player's random
observations $f_1(X_1),\ldots,f_{t-1}(X_{t-1})$.

The player's cumulative loss at the end of the game is the random
quantity $\sum_{t=1}^T f_t(X_t)$ and his expected regret against the sequence
 $f_{1:T}$ is
$$
	R(\rho_{1:T} ; f_{1:T}) 
	\eq \E \left[ \sum_{t=1}^{T} f_t(X_t) \right ]  
	~-~ \min_{x \in \K} \sum_{t=1}^{T} f_t(x) ~.
$$
The difficulty of the game is measured by its minimax regret, defined as 
$$
	\min_{\rho_{1:T}} \sup_{f_{1:T}}~ R(\rho_{1:T} ; f_{1:T}) ~.
$$

We now turn to introduce the Bayesian online optimization setting.
In the Bayesian setting, we assume that the sequence of loss functions
$F_{1:T}$, where each $F_{t} : \K \mapsto [0,1]$ is convex, is drawn
from a probability distribution $\F$ called the \emph{prior
  distribution}.  Note that $\F$ is a distribution over the entire
sequence of losses, and not over individual functions in the
sequence. Therefore, it can encode arbitrary dependencies between the
loss functions on different rounds. However, we assume that this
distribution is known to the player, and can be used to
design his policy.  The player's \emph{Bayesian
  regret} is defined as
$$
R(\rho_{1:T}; \F) \eq  
\E \left[ \sum_{t=1}^{T} F_t(X_t) - \sum_{t=1}^{T} F_t(X^{\st})\right ]  ~,
$$
where $X^{\st}$ is the point in $\K$ with the smallest cumulative
loss at the end of the game, namely the random variable
\begin{align} \label{eq:xstar}
	X^{\st} \eq \argmin_{x \in \K} \sum_{t=1}^{T} F_{t}(x) ~.
\end{align}
The difficulty of online optimization in a Bayesian
environment is measured using the maximin Bayesian regret, defined as
$$
\sup_{\F} \min_{\rho_{1:T}}~ R(\rho_{1:T}; \F) ~.
$$
In words, the maximin Bayesian regret is the regret of an optimal
Bayesian strategy over the worst possible prior $\F$.  

It turns out that the two online optimization settings we described above are closely related.
The following theorem, which is a consequence of a generalization of the von Neumann minimax theorem, shows that the minimax adversarial regret and maximin Bayesian regret are equal.

\begin{theorem} \label{thm:minimax}
It holds that
$$
	\min_{\rho_{1:T}} \sup_{f_{1:T}}~ R(\rho_{1:T} ; f_{1:T})
	\eq \sup_{\F} \min_{\rho_{1:T}}~ R(\rho_{1:T}; \F) ~.
$$
\end{theorem}

For completeness, we include a proof of this fact in \cref{sec:sion}.
As a result, instead of analyzing the minimax regret directly, we can
analyze the maximin Bayesian regret. That is, our new goal is to
design a prior-dependent player policy that guarantees a small regret
against \emph{any} prior distribution~$\F$.

\subsection{Bayesian Analysis with Full Feedback} 
\label{sec:bayes-full}

As a warm-up, we first consider the Bayesian setting where the player
receives full-feedback. Namely, on round $t$, after the player draws
a point $X_{t} \sim \pi_t$ and incurs a loss of $F_t(X_t)$, we
assume that she observes the entire loss function $F_{t}$ as
feedback. We show how minimax duality can be used to recover the known
$\O(\sqrt{T})$ regret bounds for this setting. For simplicity, we
focus on the concrete setting where $\K = \Del_{n}$ (the
$n$-dimensional simplex), and where the convex loss functions
$F_{1:T}$ are also $1$-Lipschitz with respect to the $L_{1}$-norm
(with probability one).

The evolution of the game is specified by a filtration~$\H_{1:T}$,
where each $\H_{t}$ denotes the history observed by the player up to
and including round~$t$ of the game; formally, $\H_{t}$ is the
sigma-field generated by the random variables $X_{1:t}$ and $F_{1:t}$.
To simplify notations, we use the shorthand $\E_{t}[\cdot] = \E[\,
  \cdot \mid \H_{t-1}]$ to denote expectation conditioned on the
history before round $t$. The analogous shorthands $\Pr_t(\cdot)$ and
$\var_{t}(\cdot)$ are defined similarly.

Recall that the player's policy can rely on the prior $\F$. A natural deterministic policy is to choose, based on the random variable $X^{\st}$ defined in \cref{eq:xstar}, actions  $X_{1:T}$ according to
\begin{align} \label{eq:doob}
	\forall ~ t \in [T] ~, \qquad
	X_{t} \eq \E_{t}[X^{\st} ] 
	~.
\end{align}
In other words, the player uses his knowledge of the prior and his
observations so far to calculate a posterior distribution over loss
functions, and then chooses the expected best-point-in-hindsight.
Notice that the sequence $X_{1:T}$ is a martingale (in fact, a Doob martingale), 
whose elements are vectors in the simplex.

The following lemma shows that the expected instantaneous (Bayesian)
regret of the strategy on each round $t$ can be upper bounded in terms
of the variation of the sequence $X_{1:T}$ on that round.
\begin{lemma} \label{lem:tv}
Assume that with probability one, the loss functions $F_{1:T}$ are
convex and $1$-Lipschitz with respect to some norm $\norm{\cdot}$.
Then the strategy defined in \cref{eq:doob} guarantees
$\E[F_{t}(X_{t}) - F_{t}(X^{\st})] \le \E[\norm{X_t -X_{t+1}}]$ for
all $t$.
\end{lemma}

\begin{proof}
By the subgradient inequality, we have $F_{t}(X_{t}) - F_{t}(X^{\st})
\le \nabla F_t(X_t) \cdot (X_t - X^{\st})$ for all $t$. The Lipschitz
assumption implies that $\norm{\nabla F_t(X_t)}_{*} \le 1$, where
$\norm{\cdot}_{*}$ is the norm dual of $\norm{\cdot}$. Using
\cref{eq:doob}, noting that $X_{t}, F_{t} \in \H_{t}$, and taking the conditional expectation, we get
$$
	\E_{t+1}\big[ F_{t}(X_{t}) - F_{t}(X^{\st}) \big]
	\leq \nabla F_t(X_t) \cdot (X_t - \E_{t+1}[X^{\st}])
	\eq \nabla F_t(X_t) \cdot (X_t - X_{t+1}) 
	~~.
$$
Finally, applying Holder's inequality on the right-hand side and taking expectations proves the lemma.
\end{proof}
To bound the total variation of $X_{1:T}$, we use a bound of
\cite{neyman2013maximal} on the total variation of martingales in the
simplex.
\begin{lemma}[\citealp{neyman2013maximal}]
For any martingale $Z_1, \hdots, Z_{T+1}$ in the $n$-dimensional simplex, one has 
$$ 
	\E \left[\sum_{t=1}^T \norm{Z_t - Z_{t+1}}_1 \right] 
	\leq \sqrt{\tfrac{1}{2} T\log{n}} 
	~.
$$
\label{lem:neyman}        
\end{lemma}
\cref{lem:neyman} and \cref{lem:tv} together yield a $O(\sqrt{T
  \log{n}})$ bound on the maximin Bayesian regret of online convex
optimization on the simplex with full-feedback. \cref{thm:minimax}
then implies the same bound over the minimax regret in the
corresponding adversarial setting, recovering the well-known bounds in
this case~\cite[e.g.,][]{kivinen1997exponentiated}.
We remark that essentially the same technique can be used to retrieve known
dimension-free regret bounds in the Euclidean setting, e.g., when $\K$
is an Euclidean ball and the losses are Lipschitz with respect to the
$L_{2}$ norm; in this case, the $L_{2}$ total variation of the
martingale $X_{1:T}$ can be shown to be bounded by $O(\sqrt{T})$ with
no dependence on $n$.%
\footnote{This follows from the fact that a martingale in $\reals^n$ can always be projected to a martingale in $\reals^2$ with the same magnitude of increments; 
namely, given a martingale $Z_1,Z_2,\dots$ in $\reals^n$ one can show that there exists a martingale sequence $\wt{Z}_1,\wt{Z}_2,\dots$ in $\reals^2$ such that $\|Z_t- Z_{t+1}\|_2 = \norm{\wt{Z}_t- \wt{Z}_{t+1}}_2$ for all $t$.}

\subsection{Regret Analysis of Bayesian Bandits} 
\label{sec:bayes-bandit}

The analysis in this section builds on the technique introduced by
\cite{russo2014information}. 
While their analysis is stated for prior distributions that are
i.i.d.~(namely, $\F$ is a product distribution), we show that 
it extends to arbitrary prior distributions with essentially no modifications.

We begin by restricting our attention to finite decision sets $\K$, and denote $K = \abs{\K}$. (When we get to the analysis of Bayesian bandit convex
optimization, $\K$ will be an appropriately chosen grid of points in
$[0,1]$.) 
In the bandit case, the history $\H_{t}$ is the sigma-field
generated by the random variables $X_{1:t}$ and
$F_{1}(X_{1}),\ldots,F_{t}(X_{t})$. Following \cite{russo2014information}, 
we consider the following quantities related to the filtration $\H_{1:T}$:
\begin{align} \label{eq:RtVt}
	\forall ~ x \in \K ~, \qquad
	\begin{aligned}
	\R_{t}(x) &\eq \E_{t} \big[ F_{t}(x) - F_{t}(X^{\st}) \big] ~, \\
	\V_{t}(x) &\eq \var_{t}\big( \E_{t}[F_{t}(x) \mid X^{\st}] \big) ~.
	\end{aligned}
\end{align}
The random quantity $\R_{t}(x)$ is the expected regret incurred by 
playing the point $x$ on round~$t$, conditioned on the history. 
Hence, the cumulative expected regret of the player equals 
$\E[ \sum_{t=1}^{T} \R_{t}(X_{t}) ]$.  The random variable $\V_{t}(x)$ is
a proxy for the information revealed about $X^{\st}$ by choosing the point $x$ on
round $t$. Intuitively, if the value of $F_{t}(x)$ varies significantly as a function of the random variable $X^{\st}$, then observing the value of
$F_{t}(x)$ should reveal much information on the identity of $X^{\st}$. 
(More precisely, $\V_{t}(x)$ is the amount of variance in $F_{t}(x)$ explained by the random variable $X^{\st}$.)

The following lemma can be viewed as an analogue of \cref{lem:neyman} in the bandit setting. 
\begin{lemma} \label{lem:russo}
For any player strategy and any prior distribution $\F$, it holds that 
$$\E\left[ \sum_{t=1}^{T} \sqrt{ \E_{t}[\V_{t}(X_{t})] } \right] \leq \sqrt{\thalf T \log{K}} ~.$$
\end{lemma}

The proof uses tools from information theory to relate the quantity
$\V_{t}(X_{t})$ to the \emph{decrease in entropy} of the random
variable $X^{\st}$ due to the observation on round $t$; the total
decrease in entropy is necessarily bounded, which gives the bound in
the lemma.  For completeness, we give a proof in \cref{sec:russo}.

\cref{lem:russo} suggests a generic way of obtaining regret bounds for
Bayesian algorithms: first bound the instantaneous regret
$\E_{t}[\R_{t}(X_{t})]$ of the algorithm in terms of
$\sqrt{\E_{t}[\V_{t}(X_{t})]}$ for all $t$, then sum the bounds and
apply \cref{lem:russo}.  \cite{russo2014information} refer to the
ratio $\E_{t}[\R_{t}(X_{t})] \big/\! \sqrt{\E_{t}[\V_{t}(X_{t})]}$ as the
\emph{information ratio}, and show that for Thompson Sampling over a
set of $K$ points (under an i.i.d.~prior $\F$) this ratio is always
bounded by $\sqrt{K}$, with no assumptions on the structure of the
functions $F_{1:T}$. In the sequel, we show that this $\sqrt{K}$
factor can be improved to a polylogarithmic term in $K$ (albeit using
a different algorithm) when $F_{1:T}$ are univariate convex functions.

\section{Leveraging Convexity: The Local-to-Global Lemma} \label{sec:loc-to-glob}

\begin{figure}[t]
\begin{center}
\includegraphics{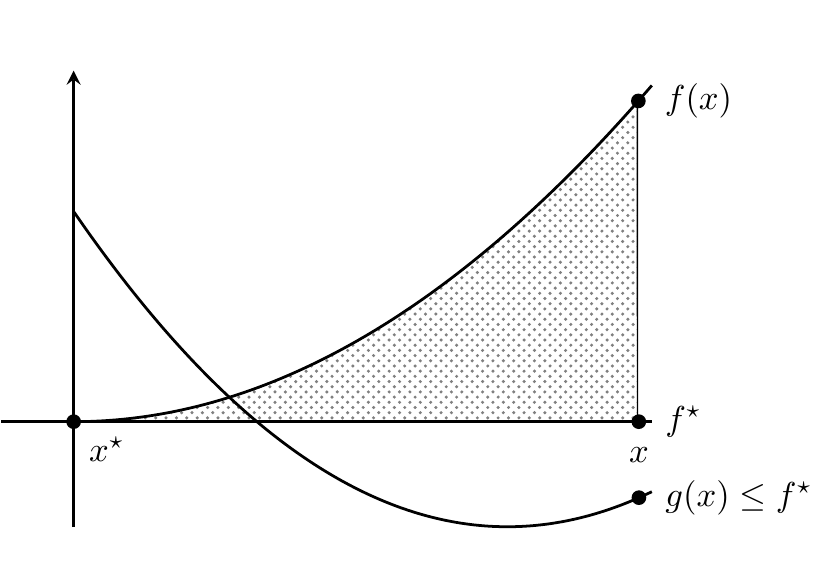}
%\begin{tikzpicture}[trim left=-0.5cm]
%\begin{axis}[
%	xmin=-1, xmax=6, ymin=-1, ymax=3,
%	xscale=1.5,
%	ticks=none,
%	axis x line=middle,
%	axis y line=middle,
%	hide x axis,
%	y axis line style={thick,shorten >=20pt, shorten <=10pt},
%	xlabel=$\smaller f^{\st}$,
%]
%% fill area
%\addplot[pattern=crosshatch dots,pattern color=gray,domain=0:3.9,samples=40] {0.15*x^2} \closedcycle
%;
%% f
%\addplot[smooth,thick,mark=none,domain=0:4,samples=40] {0.15*x^2}
%node[pos=0,circle,fill=black,inner sep=1.5pt] (xst) {} 
%node[pos=0.97,circle,fill=black,inner sep=1.5pt] (fx) {} 
%;
%% g
%\addplot[smooth,thick,mark=none,domain=0:4,samples=40] {0.25*x^2-1.5*x+1.5}
%node[pos=0.98,circle,fill=black,inner sep=1.5pt] (gx) {} 
%;
%% f^{*}
%\addplot[smooth,thick,mark=none,domain=-0.5:4,samples=40] {0}
%node[pos=0.98,circle,fill=black,inner sep=1.5pt] (x) {} 
%;
%% labels
%\node[below right = -2pt of xst] {\smaller $x^{\st}$};
%\node[below = 1pt of x] {\smaller $x$};
%\node[right = 1pt of fx] {\smaller $f(x)$};
%\node[right = 1pt of x] {\smaller $f^{\st}$};
%\node[right = 1pt of gx] {\smaller $g(x) \le f^{\st}$};
%
%\end{axis}
%\end{tikzpicture}
\end{center}
\vspace{-0.5cm}
\caption{\label{fig:l2g}
An illustration of the local-to-global lemma. The $L_{2}$ distance between the reference convex function $f$ to a convex function $g$ in the interval $[x^{\st},x]$, where $x^{\st}$ is the minimizer of $f$ and $x$ is a point such that $g(x) \le f(x^{\st})$, can be lower bounded in terms of the shaded area that depicts the energy of the function $f$ in the same interval.
}
\end{figure}

To obtain the desired regret bound, our analysis must somehow take
advantage of some special property of convex functions. In this
section, we specify which property of convex functions is leveraged in our proof.

To gain some intuition, consider the following prior distribution,
which is not restricted to convex functions: draw a point $X^\star$
uniformly in $[0,1]$ and sets all of the loss functions to be the same
function, $F_t(x) = 1\!\!1_{x \neq X^\star}$ (the indicator of $x \neq
X^\star$). Regardless of the player's policy, she will almost surely
miss the point $X^\star$, observe the loss sequence $1,\ldots,1$, and
incur a regret of $T$. The reason for this high regret is that the
prior was able to hide the good point $X^\star$ in each of the loss
functions without modifying them globally. However, if the loss
functions are required to be convex, it is impossible to design a
similar example. Specifically, any local modification to a convex
function necessarily changes the function globally (namely, at many
different points). This intuitive argument is formalized in the
following lemma; here we denote by 
$\norm{g}_{\nu}^{2} = \int g^{2} d\nu$ 
the $L_{2}$-norm of a function $g : [0,1] \mapsto \reals$ with respect 
to a probability measure $\nu$.
\begin{lemma}[Local-to-global lemma] \label{lem:loc-to-glob}
Let $f, g : [0,1] \mapsto \reals$ be convex functions. 
Denote $x^{\st} = \argmin_{x \in [0,1]} f(x)$ and $f^{\st} = f(x^{\st})$, and let $x \in [0,1]$ such that $g(x) \le f^{\st} < f(x)$. 
Then for any probability measure $\nu$ supported on~$[x^{\st},x]$, we have
$$
	\frac{\norm{f - g}_{\nu}^2}{(f(x) - g(x))^2} 
	\geq \nu(x^{\st}) \cdot 
		\frac{\norm{f - f^{\st}}_{\nu}^2}{(f(x) - f^{\st})^2} ~.
$$
\end{lemma}
To understand the statement of the lemma, it is convenient to think
of~$f$ as a reference convex function, to which we compare another
convex function~$g$; see \cref{fig:l2g}. 
If $g$ substantially differs from $f$ at one
point $x$ (in the sense that $g(x) \le f^{\st}$), then the lemma asserts that 
$g$ must also differ from $f$ globally (in the sense that $\norm{f - g}_{\nu}^2$ 
is large).

%\medskip
\begin{proof}
Let $X$ be a random variable distributed according to $\nu$. To prove the lemma, we must show that
$$
	\frac{\E (f(X) - g(X))^2}{(f(x) - g(x))^2} 
	\geq \Pr(X = x^{\st}) \cdot \frac{\E (f(X) - f^{\st})^2}{(f(x) - f^{\st})^2} 	
	~.
$$
Without loss of generality we can assume that $x > x^{\st}$. 
Let $x_0$ be the unique point such that $f(x_0) = g(x_0)$, and if such a point does not exist let $x_0=x^{\st}$. 
Note that $x_0 < x$, and observe that $g$ is below (resp.~above) $f$ on $[x_0,x]$ (resp.~$[x^{\st},x_0]$). 

\paragraph{Step 1:} 
We first prove that, without loss of generality, one can assume that $g$ is linear. Indeed consider $\tilde{g}$ to be the linear extension of the chord of $g$ between $x$ and $x_0$. 
Then we claim that:
\begin{equation} \label{eq:toprove}
	\frac{\E (f(X) - g(X))^2}{(f(x) - g(x))^2} 
	\geq \frac{\E (f(X) - \t{g}(X))^2}{(f(x) - \t{g}(x))^2}.
\end{equation}
Indeed the denominator is the same on both side of the inequality, and clearly by convexity $\tilde{g}$ is always closer to $f$ than $g$. Thus in the following we assume that $g$ is linear.

\paragraph{Step 2:} 
We show now that one can assume $g(x) = f^{\st}$. 
Let $\t{g}$ be the linear function such that $\t{g}(x) = f^{\st}$ and $\t{g}(x_0) = f(x_0)$.
Similarly to the previous step, we have to show that \cref{eq:toprove} holds true. 
We will show that 
%$$
%	h(y) \eq \frac{f(y)-g(y)}{f(y)-\t{g}(y)}
%$$ 
$
	h(y) = \big( f(y)-g(y) \big) / \big( f(y)-\t{g}(y) \big)
$ 
is non-increasing on $[x^{\st},x]$, which clearly implies \cref{eq:toprove}. 
A simple approximation argument shows that without of generality one can assume that $f$ is differentiable, in which case $h$ is also differentiable.
Observe that $h'(y)$ has the same sign as 
$
	u(y) = f'(y) (g(y) - \t{g}(y)) - g'(y) (f(y) - \t{g}(y)) + \t{g}'(y) (f(y) - g(y)) ~.
$ 
Moreover, $u'(y) = f''(y) (g(y) - \t{g}(y))$ since $g''=\t{g}''=0$, and thus $u$ is decreasing on $[x_0, x]$ and increasing on $[x^{\st},x_0]$ (recall that by convexity $f''(y) \ge 0$). 
Since $u(x_0) \le 0$ (in fact $u(x_0) = 0$ in the case $x_0 \ne x^{\st}$), this implies that $u$ is nonpositive, and thus $h$ is non-increasing, which concludes this step.

\paragraph{Step 3:} 
It remains to show that when $g$ is linear with $g(x) = f^{\st}$, then
\begin{equation} \label{eq:toprove2}
	\E (f(X) - g(X))^2 
	\geq \Pr(X = x^{\st}) \cdot \E (f(X) - f^{\st})^2 
	~.
\end{equation}
For notational convenience we assume $f^{\st}=0$. By monotonicity of $f$ and $g$ on $[x^{\st},x]$, one has $\forall y \in [x^{\st},x]$, $|f(y)-g(y)| \geq |f(y) - f(x_0)|$. Therefore, it holds that
\begin{equation} \label{eq:toprove3}
	\E (f(X) - g(X))^2 
	\geq \E(f(X) - f(x_0))^2 
	\geq \var(f(X)) 
	\eq \E f^2(X) - (\E f(X))^2 .
\end{equation}
Now using Cauchy-Schwarz one has
$
	\E f(X) 
	= \E f(X) \ind{X \neq x^{\st}} 
	\le \sqrt{\Pr(X \neq x^{\st}) \cdot \E f^2(X)} 
	,
$
which together with \cref{eq:toprove3} yields \cref{eq:toprove2}.
\end{proof}

\section{Algorithm for Bayesian Convex Bandits} 
\label{sec:algo}

In this section we present and analyze our algorithm for one-dimensional bandit convex optimization in the Bayesian setting, over $\K = [0,1]$.
Recall that in Bayesian setting, there is a prior distribution $\F$ over a sequence $F_{1:T}$ of loss functions over $\K$, such that each function $F_{t}$ is convex (but not necessarily Lipschitz) and take values in $[0,1]$ with probability one.

Before presenting the algorithm, we make the following simplification: 
given $\eps > 0$, we discretize the interval $[0,1]$ to a grid $\X_{\eps} = \set{x_{1},\ldots,x_{K}}$ of $K = 1/\eps^{2}$ equally-spaced points and treat~$\X_{\eps}$ as the de facto decision set, restricting all computations as well as the player's decisions to this finite set. We may do so without loss of generality: it can be shown (see \cref{sec:lip}) 
that for any sequence of convex loss functions $F_{1},\ldots,F_{T} : \K \mapsto [0,1]$, 
the $T$-round regret (of any algorithm) with respect to $\X_{\eps}$ is at most $2\eps T$ larger than its regret with respect to $\K$, and we will choose $\eps$ to be small enough so that this difference is negligible.

After fixing a grid $\X_{\eps}$, we introduce the following definitions. 
We define the random variable $X^{\st} = \argmin_{x \in \X_{\eps}} \sum_{t=1}^{T} F_{t}(x)$, and for all $t$ and $i, j \in [K]$ let
\begin{align} \label{eq:ffi}
\begin{alignedat}{2} 
	&\alpha_{i,t} 
&\eq& 
	\!\Pr_{t}(X^{\st} = x_{i}) ~, 
\\
	&f_{t}(x_{i}) 
&\eq& 
	\E_{t}[F_{t}(x_{i})] ~, 
\\
	&f_{j,t}(x_{i}) 
&\eq& 
	\E_{t}[F_{t}(x_{i}) \mid X^{\st}=x_{j}] ~.
\end{alignedat}
\end{align}
In words, $X^{\st}$ is the optimal action in hindsight, and $\alpha_{t} = (\alpha_{1,t},\ldots,\alpha_{K,t})$ is the posterior distribution of $X^{\st}$ on round $t$. 
The function $f_{t} : \X_{\eps} \mapsto [0,1]$ is the expected loss function on round $t$ given the feedbacks observed in previous rounds, and for each $j \in [K]$, the function $f_{j,t} : \X_{\eps} \mapsto [0,1]$ is the expected loss function on round $t$ conditioned on $X^{\st}=x_{j}$ and on the history.

\begin{figure}[t]

\RestyleAlgo{boxed}
\setlength{\algomargin}{1em}

\begin{algorithm}[H]
%\DontPrintSemicolon
\SetAlgoNoEnd
\SetAlgoNoLine
\SetArgSty{textrm}
\SetKwFor{For}{For}{}{}
\SetKwInput{KwParams}{Inputs}

\KwParams{prior distribution $\F$, tolerance parameter $\eps>0$}

\BlankLine

Let $K = 1/\eps^{2}$ and $\X_{\eps} = \set{x_{1},\ldots,x_{K}}$ with $x_{i} = i/K$ for all $i \in [K]$ \;

\For{round $t=1$ to $T$:}{
	For all $i \in [K]$, compute $\alpha_{i,t}$, $f_{t}(x_{i})$ and $f_{i,t}(x_{i})$ defined in \cref{eq:ffi} \;
	
	Find $i^{\st}_{t} = \argmin_{i} f_{t}(x_{i})$ and let $x^{\st}_{t} = x_{i^{\st}_{t}}$ \;
	
	Define the set
	\begin{align} \label{eq:S}
		S_{t} \eq \lrset{ 
			i \in [K] ~:~ \text{$f_{i,t}(x_{i}) \le f_{t}(x^{\st}_{t})$
			~~and~~
			$\alpha_{i,t} \ge \tfrac{\eps}{K}$} 
		}
		\text{ \;}
	\end{align}
	
	Sample $X_{t}$ from the distribution $\pi_{t} = (\pi_{1,t},\ldots,\pi_{K,t})$ over $\X_{\eps}$, given by
	\begin{align} \label{eq:p}
		\forall ~ i \in [K] ~,
		\qquad
		\pi_{i,t} \eq 
		\tfrac{1}{2}\alpha_{i,t} \cdot \ind{i \in S_{t}} 
			+ (1 - \tfrac{1}{2}\alpha_{t}(S_{t})) \cdot \ind{i = i^{\st}_{t}}
		~,
	\end{align}
	where we denote $\alpha_{t}(S) = \sum_{i \in S} \alpha_{i,t}$ \;
	
	Play $X_{t}$ and observe feedback $F_{t}(X_{t})$ \;
}
\end{algorithm}

%\vspace{-0.3cm}
\caption{A modified Thompson Sampling strategy that guarantees $\tO(\sqrt{T})$ expected Bayesian regret for any prior distribution $\F$ over convex functions $F_{1},\ldots,F_{T} : [0,1] \mapsto [0,1]$.} \label{fig:algo}
\end{figure}

Using the above definitions, we can present our algorithm, shown in \cref{fig:algo}.
On each round $t$ the algorithm computes, using the knowledge of the prior $\F$ and the feedback observed in previous rounds, the posterior $\alpha_{t}$ and the values $f_{t}(x_{i})$ and $f_{i,t}(x_{i})$ for all $i \in [K]$.
Also, it computes the minimizer $x^{\st}_{t}$ of the expected loss $f_{t}$ over the set $\X_{\eps}$, which is the point that has the smallest expected loss on the current round.
Instead of directly sampling the decision from the posterior $\alpha_{t}$ (as Thompson Sampling would do), we make the following two simple modifications.
First, we add a forced exploitation on the optimizer $x^{\st}_{t}$ of the expected loss to ensure that the player chooses this point with probability at least $\thalf$.
Second, we transfer the probability mass assigned by the posterior to points not represented in the set $S_{t}$, towards $x^{\st}_{t}$.
The idea is that playing a point $x_{i}$ with $i \notin S_{t}$ is useless for the player, either because it has a very low probability mass, or because playing $x_{i}$ would not be (much) more profitable to the player than simply playing $x^{\st}_{t}$ on round $t$, \emph{even if she is told that $x_{i}$ is the optimal point at the end of the game}.

The main result of this section is the following regret bound attained by our algorithm.

\begin{theorem} \label{thm:algo}
  Let $F_{1},\ldots,F_{T} : [0,1] \mapsto [0,1]$ be a sequence of convex 
  loss functions drawn from an arbitrary prior distribution $\F$. 
  For any $\eps>0$, the Bayesian regret of the algorithm described in \cref{fig:algo} 
  over $\X_{\eps}$ is upper-bounded by
\begin{align*}
%	\E \left[ \sum_{t=1}^{T} F_t(X_t) - 
%		\min_{x \in \K} \sum_{t=1}^{T} F_t(x)\right ]
%	R(\F) \leq 
	10\sqrt{T} \log\frac{2K}{\eps} + 10\eps T \sqrt{ \log\frac{2K}{\eps} }
	~.
\end{align*}
In particular, for $\eps = 1/\sqrt{T}$ we obtain an upper bound of
$
%	R(\F) = 
\O( \sqrt{T} \log{T} )
%	.
$
over the regret.
\end{theorem}

\begin{proof}
%First, we observe that restricting attention to the finite set $\K$ does not harm the regret too much (despite the fact that the loss functions are not necessarily Lipschitz).
%
%\begin{lemma} \label{lem:KtoX}
%The (Bayesian) regret of the algorithm with respect to the set $\K$ is at most $\eps T$ larger than its regret with respect to the entire $[0,1]$ interval.
%\end{lemma}
%
%See \cref{sec:lip} for a proof of this fact.
%Thus, in what follows we focus on bounding the Bayesian regret with respect to $\K$.
We bound the Bayesian regret of the algorithm (with respect to $\X_{\eps}$) on a per-round basis, via the technique described in \cref{sec:bayes-bandit}.
Namely, we fix a round $t$ and bound $\E_{t}[ \R_{t}(X_{t}) ]$ in terms of $\E_{t}[ \V_{t}(X_{t}) ]$ (see \cref{eq:RtVt}).
Since the round is fixed throughout, we omit the round subscripts from our notation, and it is understood that all variables are fixed to their state on round~$t$.

First, we bound the expected regret incurred by the algorithm on round $t$ in terms of the posterior $\alpha$ and the expected loss functions $f,f_{1},\ldots,f_{K}$.

\begin{lemma} \label{lem:Rt-bound}
With probability one, it holds that
\begin{align} \label{eq:Rt-bound}
	\E_{t}[ \R_{t}(X_{t}) ]
	\leq \sum_{i \in S} \alpha_{i} (f(x_{i}) - f_{i}(x_{i})) + \eps
	~.
\end{align}
\end{lemma}

The proofs of all of our intermediate lemmas are deferred to the end of the section.
Next, we turn to lower bound the information gain of the algorithm (as defined in \cref{eq:RtVt}).
Recall our notation $\norm{g}_{\nu}^{2}$ that stands for the $L_{2}$-norm of a function $g : \K \mapsto \reals$ with respect to a probability measure $\nu$ over $\K$; specifically, for a measure $\nu$ supported on the finite set $\X_{\eps}$ we have $\norm{g}_{\nu}^{2} = \sum_{i=1}^{K} \nu_{i} g^{2}(x_{i})$.
%also recall the definition of $\pi$ in \cref{eq:p}.

\begin{lemma} \label{lem:Vt-bound}
%The expected information gain of the algorithm on round $t$ has
With probability one, we have
\begin{align} \label{eq:Vt-bound}
	\E_{t}[ \V_{t}(X_{t}) ]
	\geq \sum_{i \in S} \alpha_{i} \norm{f-f_{i}}_{\pi}^{2}
	~.
\end{align}
%with probability one.
\end{lemma}

We now set to relate between the right-hand sides of \cref{eq:Rt-bound,eq:Vt-bound}, in a way that would allow us to use \cref{lem:russo} to bound the expected cumulative regret of the algorithm.  
In order to accomplish that, we first relate each regret term $f(x_{i}) - f_{i}(x_{i})$ to the corresponding information term $\norm{f-f_{i}}_{\pi}^{2}$.
Since $f$ and the $f_{i}$'s are all convex functions, this is given by the local-to-global lemma (\cref{lem:loc-to-glob}) which lower-bounds the global quantity $\norm{f-f_{i}}_{\pi}^{2}$ in terms of the local quantity $f(x_{i}) - f_{i}(x_{i})$.

To apply the lemma, we establish some necessary definitions.
For all $i \in S$, define~$\eps_{i} = \eps \, \abs{x_{i}-x^{\st}}$, and let $S_{i} = S \cap [x_{i},x^{\st}]$ be the neighborhood of $x_{i}$ that consists of all points in $S$ lying between~(and including) $x_{i}$ and the optimizer $x^{\st}$ of $f$.
Now, define weights $w_{i}$ for all $i \in S$ as follows: %$w_{i^{\st}} = \pi_{i^{\st}}$, and for all $i \in S$, $i \ne i^{\st}$,
\begin{align} \label{eq:w}
	\forall ~ i^{\st} \ne i \in S ~,
	\qquad
	w_{i} &\eq
	\sum_{j \in S_{i}} \pi_{j} \lr{
		\frac{  f(x_{j})-f(x^{\st}) + \eps_{j} }{ f(x_{i})-f(x^{\st}) + \eps_{i} } 
	}^{2} ~,
	\qquad\text{and}\qquad
	w_{i^{\st}} \eq \pi_{i^{\st}} 
	~.
\end{align}
With these definitions, \cref{lem:loc-to-glob} can be used to prove the following.

\begin{lemma} \label{lem:l2bound}
For all $i \in S$ it holds that
$
	\norm{f-f_{i}}_{\pi}^{2} 
	\geq \tfrac{1}{4} w_{i} ( f(x_{i}) - f_{i}(x_{i}) )^{2} - \eps^{2}
	.
$
\end{lemma}

Now, averaging the inequality of the lemma with respect to $\alpha$ over all $i \in S$ and using the fact that $\sqrt{a+b} \le \sqrt{a}+\sqrt{b}$ for any $a,b \ge 0$, we obtain
\begin{align*}
	\sqrt{ \sum_{i \in S}^{\phantom{.}} \alpha_{i} w_{i} ( f(x_{i}) - f_{i}(x_{i}) )^{2} }
	\leq 2 \sqrt{\sum_{i \in S}^{\phantom{.}} \alpha_{i} \norm{f-f_{i}}_{\pi}^{2} } 
		+ 2\eps
	~.
\end{align*}
On the other hand, the Cauchy-Schwarz inequality gives
\begin{align*}
	\sum_{i \in S} \alpha_{i} ( f(x_{i}) - f_{i}(x_{i}) )
	\leq \sqrt{\sum_{i \in S}^{\phantom{.}} \frac{\alpha_{i}}{w_{i}}}
		\cdot \sqrt{ \sum_{i \in S}^{\phantom{.}} \alpha_{i} w_{i} ( f(x_{i}) - f_{i}(x_{i}) )^{2} }
	~.
\end{align*}
Combining the two inequalities and recalling \cref{lem:Rt-bound,lem:Vt-bound}, we get
%This, together with \cref{lem:l2bound,lem:log-sum}, gives
\begin{align} \label{eq:final}
	\E_{t}[ \R_{t}(X_{t}) ]
%	\leq \sum_{i \in S} \alpha_{i} ( f(x_{i}) - f_{i}(x_{i}) )
	\leq 2\sqrt{\sum_{i \in S}^{\phantom{.}} \frac{\alpha_{i}}{w_{i}}} 
	\cdot \lr{ 
		\sqrt{ \E_{t}[\V_{t}(X_{t})] } 
		+ \eps
		}
	+ \eps
	~.
\end{align}
It remains to upper bound the sum $\sum_{i \in S} \frac{\alpha_{i}}{w_{i}}$. 
This is accomplished in the following lemma.

\begin{lemma} \label{lem:log-sum}
We have
%%%%\begin{align*}
$$
%	\sum_{i \in S} \frac{\alpha_{i}}{w_{i}}
	\sum_{i \in S} \frac{\alpha_{i}}{w_{i}}
	\leq 20\log\frac{2K}{\eps}
	~.
$$
%%%%\end{align*}
\end{lemma}
Finally, plugging the bound of the lemma into \cref{eq:final} and using \cref{lem:russo}, yields the stated regret bound.
\end{proof}

\subsection{Remaining Proofs}

We first give the proof of \cref{lem:Rt-bound}.
Recall that for readability, we omit the subscripts specifying the round number $t$ from our notation.

%\medskip
\begin{proof}%[Proof of \cref{lem:Rt-bound}]
The expected instantaneous regret can be written in terms of the distributions $\pi$ and $\alpha$, and the functions $f_{1},\ldots,f_{K}$ and $f$ as follows:
\begin{align*}
	\E_{t}[ \R_{t}(X_{t}) ]
	&\eq \sum_{i=1}^{K} \pi_{i} \, \R_{t}(x_{i}) \\
	&\eq \sum_{i=1}^{K} \pi_{i} \, \E_{t}[ F_{t}(x_{i}) ]
		- \sum_{i=1}^{K} \alpha_{i} \, \E_{t}[ F_{t}(x_{i}) \mid X^{\st}=x_{i}] \\
	&\eq \sum_{i=1}^{K} \pi_{i} f(x_{i}) - \sum_{i=1}^{K} \alpha_{i} f_{i}(x_{i})
	~.
\end{align*}
Next, we consider the first sum in the right-hand size of the above, that corresponds to the expected loss incurred by the algorithm.
Since $\pi$ is obtained from $\alpha$ by transferring probability mass towards~$x^{\st}$ (whose loss is the smallest), the expected loss of the algorithm has
\begin{align*}
	\sum_{i=1}^{K} \pi_{i} f(x_{i})
	&\eq \thalf \sum_{i \in S} \alpha_{i} f(x_{i}) + (1 - \tfrac{1}{2}q(S)) f(x^{\st}) \\
	&\leq \sum_{i \in S} \alpha_{i} f(x_{i})
		+ (1 - q(S)) f(x^{\st}) \\
	&\eq \sum_{i \in S} \alpha_{i} f(x_{i}) + \sum_{i \notin S} \alpha_{i} f(x^{\st}) 
	~.
\end{align*}
Also, since for each $i \notin S$ we either have $\alpha_{i} < \tfrac{\eps}{K}$ or $f(x^{\st}) - f_{i}(x_{i}) < 0$ (while both quantities are trivially bounded by $1$), 
\begin{align*}
	\sum_{i \notin S} \alpha_{i} (f(x^{\st}) - f_{i}(x_{i}))
	\leq \eps
	~.
\end{align*}
Hence, for the regret we have
\begin{align*}
	\E_{t}[ \R_{t}(X_{t}) ]
	&\eq \sum_{i=1}^{K} \pi_{i} f(x_{i}) - \sum_{i=1}^{K} \alpha_{i} f_{i}(x_{i}) \\
	&\leq \sum_{i \in S} \alpha_{i} (f(x_{i}) - f_{i}(x_{i}))
		+ \sum_{i \notin S} \alpha_{i} (f(x^{\st}) - f_{i}(x_{i})) \\
%	&\leq \sum_{i \in S} \alpha_{i} (f(x_{i}) - f_{i}(x_{i})) + \eps \\
	&\leq \sum_{i \in S} \alpha_{i} (f(x_{i}) - f_{i}(x_{i})) + \eps
	~. \qedhere
\end{align*}
\end{proof}

Next, we prove \cref{lem:Vt-bound}.

%\medskip
\begin{proof}%[Proof of \cref{lem:Vt-bound}]
The expected instantaneous information gain can be written as
\begin{align*}
	\E_{t}[ \V_{t}(X_{t}) ]
	&\eq \sum_{i=j}^{K} \pi_{j} \, 
		\var_{t}\big( \E_{t}[F_{t}(x_{j}) \mid X^{\st}] \big) \\
	&\eq \sum_{i=1}^{K}\sum_{j=1}^{K} \alpha_{i} \pi_{j} \, \big( 
		\E_{t}[F_{t}(x_{j}) \mid X^{\st}=x_{i}] - \E_{t}[F_{t}(x_{j})]
		\big)^{2} \\
	&\eq \sum_{i=1}^{K} \sum_{j=1}^{K} \alpha_{i} \pi_{j} (f_{i}(x_{j}) - f(x_{j}))^{2}
	~.
\end{align*}
The lemma then follows from
\begin{align*}
	\sum_{i=1}^{K} \sum_{j=1}^{K} 
		\alpha_{i} \pi_{j} (f(x_{j}) - f_{i}(x_{j}))^{2}
	\geq \sum_{i \in S} \alpha_{i} \sum_{j=1}^{K} \pi_{j} (f(x_{j}) - f_{i}(x_{j}))^{2} 
	\eq \sum_{i \in S} \alpha_{i} \norm{f-f_{i}}_{\pi}^{2} 
	~. &\qedhere
\end{align*}
%where the final equality holds since $\pi_{i} = \frac{1}{2} \alpha_{i}$ for all $i \in S$.
\end{proof}

%We now give the proofs of the lemmas used in our proof of the main theorem. We begin with \cref{lem:Rt-bound}.
%Due to space constraints, we defer the proof of \cref{lem:Rt-bound,lem:Vt-bound} to \cref{sec:omit}, and turn to prove \cref{lem:l2bound}. 
We now turn to prove \cref{lem:l2bound}. 
The proof uses the local-to-global lemma (\cref{lem:loc-to-glob}) discussed earlier in \cref{sec:loc-to-glob}.

%\medskip
\begin{proof}%[Proof of \cref{lem:l2bound}]
The lemma holds trivially for $i=i^{\st}$, as we defined $w_{i^{\st}} = \pi_{i^{\st}}$, whence
\begin{align*}
	\norm{f-f_{i^{\st}}}_{\pi}^{2}
	\geq \pi_{i^{\st}} ( f(x^{\st}) - f_{i^{\st}}(x^{\st}) )^{2}
	\geq \tfrac{1}{4} w_{i^{\st}} ( f(x^{\st}) - f_{i^{\st}}(x^{\st}) )^{2} - \eps^{2}
	~.
\end{align*}
Therefore, in what follows we assume that $i \in S$ and $i \ne i^{\st}$.

Consider a regularized version of the function $f$, given by $f_{\eps}(x) = f(x) + \eps\abs{x-x^{\st}}$.
Notice that $f_{\eps}$ is convex, and has a unique minimum at $x^{\st}$ with $f_{\eps}(x^{\st}) = f(x^{\st})$.
Since $\pi_{i^{\st}} \ge \half$ by construction (the algorithm exploits with probability $\half$), and for all $i \in S$ we have $f_{i}(x_{i}) \le f_{\eps}(x^{\st}) < f_{\eps}(x_{i})$, we can apply \cref{lem:loc-to-glob} to the functions $f_{\eps}$ and $f_{i}$ and obtain
\begin{align*}
	\frac{ \sum_{j \in S_{i}} \pi_{j} ( f_{\eps}(x_{j}) - f_{i}(x_{j}) )^{2} }
		{ ( f_{\eps}(x_{i}) - f_{i}(x_{i}) )^{2} }
	\geq
	\half \cdot
	\frac{ \sum_{j \in S_{i}} \pi_{j} (f_{\eps}(x_{j})-f_{\eps}(x^{\st}))^{2} }
		{ (f_{\eps}(x_{i})-f_{\eps}(x^{\st}))^{2} } 
	~.
\end{align*}
Now, notice that $f_{\eps}(x_{j}) - f_{\eps}(x^{\st}) = f(x_{j}) - f(x^{\st}) + \eps \, \abs{x_{j} - x^{\st}}$ for all $j$; hence, recalling \cref{eq:w}, the right-hand side above equals $\half w_{i}$.
Rearranging and using $\norm{f_{\eps}-f_{i}}_{\pi}^{2} \ge \sum_{j \in S_{i}} \pi_{j} ( f_{\eps}(x_{j}) - f_{i}(x_{j}) )^{2}$ gives
\begin{align} \label{eq:l2bound1}
	\norm{f_{\eps}-f_{i}}_{\pi}^{2}
	\geq \thalf w_{i} ( f_{\eps}(x_{i}) - f_{i}(x_{i}) )^{2} 
	~.
\end{align}
To obtain the lemma from \cref{eq:l2bound1}, observe that by the triangle inequality,
\begin{align*}
	\norm{f-f_{i}}_{\pi}
	\geq \norm{f_{\eps}-f_{i}}_{\pi} - \norm{f_{\eps}-f}_{\pi}
	\geq \norm{f_{\eps}-f_{i}}_{\pi} - \eps 
	~,
\end{align*}
so using $(a+b)^{2} \le 2(a^{2}+b^{2})$ we can upper bound the left-hand side of \cref{eq:l2bound1} as $\norm{f_{\eps}-f_{i}}_{\pi}^{2} \le 2\norm{f-f_{i}}_{\pi}^{2} + 2\eps^{2}$.
On the right hand-side of the same inequality, we can use the lower bound $( f_{\eps}(x_{i}) - f_{i}(x_{i}) )^{2} \ge ( f(x_{i}) - f_{i}(x_{i}) )^{2}$ that follows from $f_{\eps}(x_{i})-f_{i}(x_{i}) \ge f(x_{i})-f_{i}(x_{i}) \ge 0$.
Combining these observations with \cref{eq:l2bound1}, we have
%%%%\begin{align*}
$
	\norm{f-f_{i}}_{\pi}^{2}
	\ge \tfrac{1}{4} w_{i} ( f(x_{i}) - f_{i}(x_{i}) )^{2} - \eps^{2}
	,
%	~,
$
%%%%\end{align*}
which concludes the proof.
\end{proof}
Finally, we prove \cref{lem:log-sum}.

%\medskip
\begin{proof}%[Proof of \cref{lem:log-sum}]
Since $\alpha_{i} \le 2\pi_{i}$ for all $i \in S$, it is enough to bound the sum $\sum_{i \in S} \tfrac{\pi_{i}}{w_{i}}$.
We decompose this sum into three disjoint parts: the term corresponding to $i=i^{\st}$ (in case $i^{\st} \in S$) that equals $1$ as $w_{i^{\st}} = \pi_{i^{\st}}$ by definition, a sum over the indices $i \in S$ such that $x_i < x^{\st}$, and a sum over those such that $x_i > x^{\st}$. The proof is identical for both the latter sums, thus we focus on the set $S'$ of indices such that $x_i > x^{\st}$.
Up to reindexing, we can assume that $S' = \set{1,\ldots,K'}$ for some $K' \le K$, and the corresponding points are such that $x^{\st} < x_1 < \ldots < x_{K'}$.
By our definition of $w_{i}$ (see \cref{eq:w}), we have
\begin{align*}
	\forall ~ i \in S' ~,
	\qquad
	\frac{\pi_{i}}{w_{i}}
	\eq \frac{ \pi_{i} (f(x_{i})-f(x^{\st})+\eps_{i})^{2} }
		{ \sum_{j=1}^{i} \pi_{j} (f(x_{j})-f(x^{\st})+\eps_{j})^{2} }
	~.
\end{align*}
Observe that for all $i \in S'$ it holds that $\eps_{i} = \eps \, \abs{x_{i}-x^{\st}} \ge \tfrac{\eps}{K}$, as the points $x_{1},\ldots,x_{K}$ lie on an equally-spaced grid of the interval (and $x_{i} \ne x^{\st}$ since $i^{\st} \notin S'$).
Recall also that by construction $\pi_{i} \ge \thalf \alpha_{i} \ge \tfrac{\eps}{2K}$ for all $i \in S$.
Hence, we have 
$$
	\forall ~ i \in S' ~,
	\qquad
	\thalf (\tfrac{\eps}{K})^{3}
	\leq \pi_{i} (f(x_{i})-f(x^{\st})+\eps_{i})^{2} 
	\leq 4 \pi_{i}
	~.
$$
Now, denote $\beta_{i} = \sum_{j=1}^{i} \pi_{j} (f(x_{j})-f(x^{\st})+\eps_{j})^{2}$, for which $\thalf (\tfrac{\eps}{K})^{3} \le \beta_{1} \le \ldots \le \beta_{K'} \le 4$.
Thus, we have
\begin{align*}
	\sum_{i=1}^{K'} \frac{\pi_{i}}{w_{i}}
	\eq 1 + \sum_{i=2}^{K'} \frac{\beta_{i} - \beta_{i-1}}{\beta_{i}}
	\eq 1 + \sum_{i=2}^{K'} \lr{ 1 - \frac{\beta_{i-1}}{\beta_{i}} }
	\leq 1 + \sum_{i=2}^{K'} \log\frac{\beta_{i}}{\beta_{i-1}}
	\eq 1 + \log\frac{\beta_{K'}}{\beta_{1}}
	~,
\end{align*}
where the inequality follows from the fact that $\log{z} \le z-1$ for $0 < z \le 1$.
Since $\beta_{K'}/\beta_{1} \le (\tfrac{2K}{\eps})^{3}$, we can bound the right-hand side by $1+3\log\tfrac{2K}{\eps}$.
The lemma now follows from applying the same bound to the other part of the total sum (over the indices $i$ such that $x_i < x^{\st}$) and recalling the possible term corresponding to $i=i^{\st}$.
\end{proof}

\section{Discussion and Open Problems} \label{sec:conc}

We proved that the minimax regret of adversarial one-dimensional
bandit convex optimization is $\wt O(\sqrt{T})$ by designing an
algorithm for the analogous Bayesian setting and then using minimax
duality to upper-bound the regret in the adversarial setting. Our
work raises interesting open problems.
The main open problem is whether one can generalize our analysis from
the one-dimensional case to higher dimensions (say, even $n=2$).
While much of our analysis generalizes to higher dimensions, the
key ingredient of our proof, namely the local-to-global lemma 
(\cref{lem:loc-to-glob}) is inherently one-dimensional.
%
%While our bound in the adversarial setting is non-constructive, 
We hope that the components of our analysis, and especially the
local-to-global lemma, will inspire the design of efficient algorithms
for adversarial bandit convex optimization, even though
our end result is a non-constructive bound.

The Bayesian algorithm used in our analysis is a modified version of
the classic Thompson Sampling strategy. A second open question is
whether or not the same regret guarantee can be obtained by vanilla
Thompson Sampling, without any modification. However, if it turns out
that unmodified Thompson Sampling is sufficient, the proof is likely
to be more complex: our analysis is greatly simplified by the
observation that the instantaneous regret of our algorithm is
controlled by its instantaneous information gain on each and every 
round---a claim that does not hold for Thompson Sampling.
%We defer the formal details
%of this argument to the full version of the paper, and leave the
%analysis of Thompson Sampling with convex functions for future work.

Finally, we note that our reasoning together with Proposition 5 of
\cite{russo2014information} allows to recover effortlessly Theorem 4
of \cite{BCK12}, which gives the worst-case minimax regret for online
linear optimization with bandit feedback on a discrete set in
$\reals^n$. It would be interesting to see if this proof strategy also
allows to exploit geometric structure of the point set. For instance,
could the techniques described here give an alternative proof of
Theorem 6 of \cite{BCK12}?

%\tk{discuss possible implications to zero-order (stochastic) optimization}

\subsection*{Acknowledgements}

We thank Ronen Eldan and Jian Ding for helpful discussions during the early stages of this work.

%\clearpage
\bibliographystyle{abbrvnat}
\bibliography{bco}

%\clearpage
\appendix

\section{Proof of \texorpdfstring{\cref{thm:minimax}}{}}
\label{sec:sion}

The proof relies on Sion's generalization of von Neumann's minimax theorem, which we state here for completeness (see Corollary~3.3 of \citealp{sion1958general}, or \citealp{Komiya1988}).

\begin{theorem}%[\citealp{sion1958general}]
Let $X$ and $Y$ be convex sets in two linear topological spaces, and suppose that $X$ is compact.
Let $f$ be a real-valued function on $X \times Y$ such that
\begin{enumerate}[label=(\roman*),nosep]
\item
$f(x,\cdot)$ is upper semicontinuous and  concave on $Y$ for each $x \in X$;
\item
$f(\cdot,y)$ is lower semicontinuous and  convex on $X$ for each $y \in Y$.
\end{enumerate}
Then,
$$
\min_{x \in X} \sup_{y \in Y} f(x,y)
\eq
\sup_{y \in Y} \min_{x \in X} f(x,y)
~.
$$
\end{theorem}

\begin{proof}[Proof of \cref{thm:minimax}]
For a metric space $A$ we denote by $\Del(A)$ the set of Borel probability measures on $A$.
Let $\mathcal{C}$ be the space of convex functions from the compact $\K$ to $[0,1]$. A deterministic player's strategy is specified by a sequence of operators $a_1, \hdots, a_T$, where in the full information case $a_s : \mathcal{C}^{s-1} \rightarrow \K$, and in the bandit case $a_s : [0,1]^{s-1} \rightarrow \K$. We denote by $\mathcal{A}$ the set of such sequences of operators, which is compact in the product topology. The minimax regret can be written as:
\begin{equation} \label{eq:minimax}
\min_{u \in \Del(\mathcal{A})}
\sup_{f_{1:T} \in \mathcal{C}^{T}}
\E [R_T] ~,
\end{equation}
where $R_{T}$ denotes the induced $T$-round regret, and the expectation is with respect to the random draw of a player's strategy from $u$.
Using Sion's minimax theorem, we deduce
%\footnote{It is easy to verify that the conditions of Sion's minimax theorem are satisfied. For example $\mathcal{C}$ is compact by Arzela-Ascoli, and thus by Levy-Prokhorov $\Del(\mathcal{C})$ is compact for the weak topology.},
\cref{eq:minimax} is equal to
\begin{equation} \label{eq:maximin}
\sup_{\F \in \Del(\mathcal{C}^T)}
\min_{u \in \Del(\mathcal{A})}
\E [R_T] \, ,
\end{equation}
where the expectation is with respect to both the random draw of a player's strategy from $u$, and the random draw of the sequence of losses from $\F$. 
%In the above display $\nu$ can be naturally interpreted as a prior distribution, and the player's strategy $u$ can be adapted to this prior. For this reason we refer to \eqref{eq:maximin} as the Bayesian regret. 
%Note that in \eqref{eq:maximin} one could replace $\Del(\mathcal{A})$ by $\mathcal{A}$, but we find it easier to analyze randomized strategies even in the Bayesian context (at least for the bandit case). 
Finally, to convert the statement above to the statement of \cref{thm:minimax}, we  invoke Kuhn's theorem on the payoff equivalence of behavioral strategies to general randomized strategies. More precisely, we apply the continuum version of this theorem, established by \cite{aumann64}.
%(see \cite{MSZ2013}) and the existence of regular conditional distributions (see \cite{Durrett}).
%Using Sion's minimax theorem%
%\footnote{It is easy to verify that the conditions of Sion's minimax theorem are satisfied. For example $\mathcal{C}$ is compact by Arzela-Ascoli, and thus by Levy-Prokhorov $\Del(\mathcal{C})$ is compact for the weak topology.},
%\eqref{eq:minimax} is equal to
%\begin{equation} \label{eq:maximin}
%\sup_{\nu \in \Del(\mathcal{C}^T)}
%\inf_{u \in \Del(\mathcal{A})}
%\E R_T ,
%\end{equation}
%where the expectation is with respect to both the random draw of a player's strategy from $u$, and the random draw of the sequence of losses from $\nu$. In the above display $\nu$ can be naturally interpreted as a prior distribution, and the player's %strategy $u$ can be adapted to this prior. For this reason we refer to \eqref{eq:maximin} as the Bayesian regret. Note that in \eqref{eq:maximin} one could replace $\Del(\mathcal{A})$ by $\mathcal{A}$, but we find it easier to analyze randomized %strategies even in the Bayesian context (at least for the bandit case).
\end{proof}

\newcommand{\dist}[1]{\mathbb{#1}}

\section{Information Theoretic Analysis of Bayesian Algorithms} 
\label{sec:russo}

In this section we prove \cref{lem:russo}, restated here.

\begin{repthm}[\citealp{russo2014information}]{lem:russo}
For any player strategy and any prior distribution~$\F$, it holds that 
$$\E\left[ \sum_{t=1}^{T} \sqrt{ \E_{t}[\V_{t}(X_{t})] } \right] \leq \sqrt{\thalf T \log{K}} ~.$$
\end{repthm}

The proof follows the analysis of \cite{russo2014information}.
For the proof, we require the following definition.
%\tk{TODO: introduce info theory tools used in the proof}
Let
$$
	\forall ~ x \in \K ~, \qquad
	I_{t}(x) 
	\eq \info_{t}(F_{t}(x) ; X^{\st})
$$
be the mutual information between $X^{\st}$ and the player's loss on round $t$ upon choosing the action~$x \in \K$, conditioned on the history $\H_{t-1}$ (thus, $I_{t}(x)$ is a random variable, measurable with respect to $\H_{t-1}$).
Intuitively, $I_{t}(x)$ is the expected amount of information on $X^{\st}$ revealed by playing $x$ on round $t$ of the game and observing the feedback $F_{t}(x)$.
%
%With this definition, we can prove the following.

Before proving \cref{lem:russo}, we first show an analogous claim for the information terms $I_{t}(X_{t})$.

\begin{lemma} \label{lem:russo1}
%For any sequence of random variables $X_{1:T}$ adapted to the filtration $\H_{1:T}$, it holds that
We have
$$
	\E\left[ \sum_{t=1}^{T} \sqrt{\E_{t}[I_{t}(X_{t})]} \right]
	\leq \sqrt{T \log{K}} 
	~.
$$
\end{lemma}

\begin{proof}
Let us examine how the entropy of the random variable $X^{\st}$ evolves during the game as the player gathers the observations $F_{1}(X_{1}),\ldots,F_{T}(X_{T})$.
Denoting by $\ent_{t}(\cdot)$ the entropy conditional on $\H_{t-1}$,
%\footnote{Note that $\ent_{t}(\cdot)$ is a random variable that depends on the history $\H_{t-1}$, and is not equivalent to $\ent(\cdot \mid \H_{t-1})$ being the entropy of the conditional distribution (a deterministic quantity).}
we have by standard information theoretic relations,
\begin{align*}
	I_{t}(x)
	\eq \info_{t}(F_{t}(x) ; X^{\st})
	\eq \E_{t}[ \ent_{t}(X^{\st}) - \ent_{t+1}(X^{\st}) \mid X_{t}=x ]
\end{align*}
for all points $x \in \K$. 
Thus,
\begin{align*}
	\E_{t}[I_{t}(X_{t})]
	\eq \E_{t}[ \ent_{t}(X^{\st}) - \ent_{t+1}(X^{\st}) ]
	~.
\end{align*}
Summing over $t$ and taking expectations, we obtain
\begin{align*}
	\E\left[ \sum_{t=1}^{T} \E_{t}[I_{t}(X_{t})] \right]
%	\eq \E\left[ \sum_{t=1}^{T} 
%		\E_{t}[\ent(\alpha_{t}) - \ent(\alpha_{t+1})] 
%		\right]
	\eq \sum_{t=1}^{T} \E[\ent_{t}(X^{\st}) - \ent_{t+1}(X^{\st})]
	\leq \E[ \ent_{1}(X^{\st}) ]
	\eq \ent(X^{\st})
	~.
\end{align*}
Using Cauchy-Schwarz and the concavity of the square root yields 
\begin{align*}
	\EE{ \sum_{t=1}^{T} \sqrt{\E_{t}[I_{t}(X_{t})]} }
%	\leq \sqrt{T} \cdot \EE{\sqrt{ \sum_{t=1}^{T} \E_{t}[I_{t}(X_{t})] }} 
	\leq \sqrt{T} \cdot \sqrt{\EE{ \sum_{t=1}^{T} \E_{t}[I_{t}(X_{t})] }} 
	\leq \sqrt{\ent(X^{\st}) \, T} ~.
\end{align*}
Recalling that the entropy of any random variable supported on $K$ atoms is upper bounded by $\log{K}$, the lemma follows.
\end{proof}

We can now prove \cref{lem:russo}.

%\medskip
\begin{proof}[Proof of \cref{lem:russo}]
Let $\alpha_{t} \in \Del(\K)$ be the posterior distribution of $X^{\st}$ given $\H_{t-1}$, with $\alpha_{t,x} = \Pr_{t}(X^{\st} = x)$ for all $x \in \K$. 
By the definition of mutual information, for all $x \in \K$,
\begin{align*}
	I_{t}(x)
	\eq \info_{t}(F_{t}(x) ; X^{\st})
	\eq \sum_{y \in \K} \alpha_{t,y} \, \kl{\dist{Q}_{x}}{\dist{Q}_{x \mid y}}
	~,
\end{align*}
where $\dist{Q}_{x}$ is the distribution of $F_{t}(x)$ conditioned on $\H_{t-1}$, and $\dist{Q}_{x \mid y}$  is the distribution of $F_{t}(x)$ conditioned on $\H_{t-1}$ and the event $X^{\st} = y$.
Applying Pinsker's inequality on each term on the right-hand side of the above, we obtain %\tk{there's an extra step to be explained here}
\begin{align*}
	\thalf I_{t}(x)
	&\geq \sum_{y \in \K} \alpha_{t,y} 
		\big( \E_{\dist{Q}_{x \mid y}}[F_{t}(x)] - \E_{\dist{Q}_{x}}[F_{t}(x)] \big)^{2} \\
	&\eq \sum_{y \in \K} \alpha_{t,y} 
		\big( \E_{t}[F_{t}(x) \mid X^{\st}=y] - \E_{t}[F_{t}(x)] \big)^{2} \\
	&\eq \var_{t} \big( \E_{t}[F_{t}(x) \mid X^{\st}] \big) \\
	&\eq \V_{t}(x)
	~.
\end{align*}
Hence, $\V_{t}(x) \le \half I_{t}(x)$ for all $x \in \K$, which implies that $\E_{t}[\V_{t}(X_{t})] \le \half \E_{t}[I_{t}(X_{t})]$ with probability one.
Combining this with \cref{lem:russo1}, the result follows.
\end{proof}

\section{Effective Lipschitz Property of Convex Functions} 
\label{sec:lip}

In this section we show that any convex function is essentially Lipschitz, and justify our simplifying discretization made in \cref{sec:algo}.
%explain how this implies \cref{lem:KtoX} used in our proof of \cref{thm:algo}.
The required property is summarized in the following lemma.

\begin{lemma} \label{lem:discret}
Let $\K \subseteq \reals^{n}$ be a convex set that contains a ball of radius $r$, and let $f$ be a convex function over $\K$ that takes values in $[0,1]$.
Then for a $\del$-net $\X$ of $\K$ with $\del \le \tfrac{1}{4}r \eps^{2}$, it holds that $\min_{x \in \X} f(x) \le \min_{x \in \K} f(x) + \eps$.
\end{lemma}

In particular, for the unit interval $[0,1]$ it is enough to take a grid with $K = \tfrac{4}{\eps^{2}}$ equally-spaced points, to have an $\eps$-approximation to the optimum of any convex function over $[0,1]$ taking values in $[0,1]$.
%; this directly implies \cref{lem:KtoX}.
In fact, to obtain the same $\eps$-approximation property it is enough to use a more compact grid of size $\O(\tfrac{1}{\eps}\log\tfrac{1}{\eps})$, whose points are not equally spaced; see \cref{sec:constructive} below for more details.

\cref{lem:discret} is a consequence of the following simple property of convex functions, observed by \cite{FKM05}.

\begin{lemma} \label{lem:flaxman}
Let $\K \subseteq \reals^{n}$ be a convex set that contains a ball of radius $r$ centered at the origin, and denote $\K_{\eps} = (1-\eps) \K$.
Let $f$ be a convex function over $\K$ such that $0 \le f(x) \le C$ for all $x \in \K$.
Then
\begin{enumerate}[label=(\roman*),nosep]
\item
for any $x \in \K_{\eps}$ and $y \in \K$ it holds that $\abs{f(x)-f(y)} \le \frac{C}{r\eps} \norm{x-y}$;
\item
$\min_{x \in \K_{\eps}} f(x) \le \min_{x \in \K} f(x) + C\eps$.
\end{enumerate}
\end{lemma}

\begin{proof}[Proof of \cref{lem:discret}]
Via a simple shift of the space, we can assume without loss of generality that $\K$ contains a ball of radius $r$ centered at the origin.
Let $z = \argmin_{x \in \K} f(x)$ and $y = \argmin_{x \in \K'}$, where $\K' = (1-\tfrac{\eps}{2}) \K$. 
By the definition of the $\del$-net $\X$, there exists a point $x \in \X$ for which $\norm{x-y} \le \del$.
Since $y \in \K'$ and $x \in \K$, part (i) of \cref{lem:flaxman} shows that $f(x)-f(y) \le \tfrac{2}{r\eps} \del \le \tfrac{\eps}{2}$.
On the other hand, part (ii) of the same lemma says that $f(y)-f(z) \le \tfrac{\eps}{2}$. 
Combining the inequalities we now get $f(x) \le f(z) + \eps = \min_{x \in \K} f(x) + \eps$, which gives the lemma.
\end{proof}

\section{Constructive Upper Bound in One Dimension} 
\label{sec:constructive}

Here we describe an explicit and efficient one-dimensional algorithm for bandit convex optimization with general (possibly non-Lipschitz) convex loss functions over $\K = [0,1]$, whose regret performance is better than the general $\tO(T^{5/6})$ bound of \cite{FKM05} that applies in an arbitrary dimension.
The algorithm is based on the \textsc{Exp3} strategy for online learning with bandit feedback over a finite set of $K$ points (i.e., arms), whose expected regret is bounded by $\tO(\sqrt{TK})$; see \cite{ACFS03} for further details on the algorithm and its analysis.

In order to use \textsc{Exp3} in our continuous setting, where the decision set is $[0,1]$, we form an appropriate discretization of the interval.
It turns out that using a uniform, equally-spaced grid is suboptimal and can only give an algorithm whose expected regret is of order $\tO(T^{3/4})$. 
Nevertheless, by using a specially-tailored grid of the interval we can obtain an improved $\tO(T^{2/3})$ bound; this customized grid is specified in the following lemma.

\begin{lemma} \label{lem:sebgrid}
For $0 < \eps \le 1$, define $x_{k} = \eps (1+\eps)^{k}$ for all $k \ge 0$. 
Then the set $\X_{\eps} = \set{x_{k},1-x_{k}}_{k=0}^{\infty} \cap [0,1]$ satisfies:
\begin{enumerate}[label=(\roman*),nosep]
\item
$\abs{\X_{\eps}} \le \tfrac{4}{\eps}\log\frac{1}{\eps}$;
\item
for any convex function $f : [0,1] \mapsto [0,1]$, we have $\min_{x \in \X} f(x) \le \min_{x \in [0,1]} f(x) + 2\eps$.
\end{enumerate}
\end{lemma}

\begin{proof}
To see the first claim, note that for $k > \frac{2}{\eps}\log\frac{1}{\eps}$ we have
$
	x_{k} 
	= \eps (1+\eps)^{k}
	\ge \eps \exp(\thalf k\eps)
	> 1 ,
$
where we have used the fact that $e^{x} \le 1+2x$ for $0\le x \le 1$.

Next, we prove that for any $y \in [\eps,1-\eps]$, there exists $x \in \X_{\eps}$ such that $\abs{f(y)-f(x)} \le \eps$; this would imply our second claim, as by \cref{lem:flaxman} the minimizer of $f$ over $[\eps,1-\eps]$ can only be $\eps$ larger than its minimizer over the entire $[0,1]$ interval.
We focus on the case $y \in (0,\half]$; the case $y \in [\half,1)$ is treated similarly.
Then, we have $y \in [y,1-y]$ so \cref{lem:flaxman} shows that for any $x \in [0,1]$ we have 
\begin{align} \label{eq:sebgrid}
\abs{f(x)-f(y)} 
\leq \tfrac{1}{y} \abs{x-y} 
\eq \abs{\tfrac{x}{y}-1} 
~.
\end{align}
Now, let $k$ be the unique natural number such that $x_{k} \le y \le x_{k+1}$. 
Notice that $1 \le x_{k+1}/y \le 1+\eps$, since $x_{k+1}/x_{k} = 1+\eps$.
Hence, setting $x=x_{k+1}$ in \cref{eq:sebgrid} yields $\abs{f(x)-f(y)} \le \abs{\tfrac{x}{y}-1} \le \eps$, as required.
\end{proof}

In view of the lemma, the algorithm we propose is straightforward: given a parameter $\eps>0$, form a grid $\X_{\eps}$ of the interval $[0,1]$ as described in the lemma, and execute the \textsc{Exp3} algorithm over the finite set $\X_{\eps}$.

\begin{theorem}
The algorithm described above with $\eps = T^{-1/3}$ guarantees $\tO(T^{2/3})$ regret against any sequence $f_{1:T}$ of convex (not necessarily Lipschitz) functions over $\K = [0,1]$ taking values in~$[0,1]$.
\end{theorem}

\begin{proof}
Let $x_{1:T}$ be the sequence of points from $\X_{\eps}$ chosen by \textsc{Exp3}.
By the regret guarantee of \textsc{Exp3}, we have
\begin{align*}
	\E\left[ \sum_{t=1}^{T} f_{t}(x_{t})  ~-~
	\min_{x \in \X_{\eps}} \sum_{t=1}^{T} f_{t}(x) \right]
	\eq \tO(\sqrt{TK})
	\eq \tO\lrbig{ \sqrt{\tfrac{1}{\eps} T} }
	~,
\end{align*}
where $K = \abs{\X_{\eps}}$, that according to \cref{lem:sebgrid} has $K \le \tfrac{4}{\eps}\log\tfrac{1}{\eps}$.
On the other hand, \cref{lem:sebgrid} also ensures that
\begin{align*}
	\min_{x \in \X_{\eps}} \frac{1}{T} \sum_{t=1}^{T} f_{t}(x)
	\leq \min_{x \in \K} \frac{1}{T} \sum_{t=1}^{T} f_{t}(x) + 2\eps
	~.
\end{align*}
Combining the inequalities we get the following regret bound with respect to the entire $\K$:
\begin{align*}
	\E\left[ \sum_{t=1}^{T} f_{t}(x_{t})  ~-~
	\min_{x \in \K} \sum_{t=1}^{T} f_{t}(x) \right]
	\eq \tO\lrBig{ \sqrt{\tfrac{1}{\eps} T} + \eps T }
	~.
\end{align*}
Finally, choosing $\eps = T^{-1/3}$ gives the theorem.
\end{proof}

%\section{Difficulties in Analyzing Thompson Sampling} 
%\label{sec:ts}
%
%\tk{TODO: give the parabola example, possibly not for first arxiv version}
%
%Here we show that there exists a posterior distribution $\alpha$ and expected loss functions $f,f_{1},\ldots,f_{K} : \X \mapsto [0,1]$ (on some fixed round $t$) such that for Thompson Sampling, the size $\E_{t}[\R_{t}(X_{t})]$ can be large while $\E_{t}[\V_{t}(X_{t})]$ is small. 
%Thereafter we fix a round $t$ and omit the round number from all notation.
%
%For Thompson Sampling, that samples points according to $\pi = \alpha$, we can show that
%\begin{align*}
%	\E_{t}[\R_{t}(X_{t})] 
%	&\eq \sum_{i=1}^{K} \alpha_{i} \big( f(x_{i}) - f_{i}(x_{i}) \big) ~, \\
%	\text{and}\qquad
%	\E_{t}[\V_{t}(X_{t})] 
%	&\eq \sum_{i=1}^{K} \alpha_{i} \norm{ f - f_{i} }_{\alpha}^{2}
%\end{align*}
%
%Now, consider the function $f(x) = x^{2}$. Take a uniform posterior, i.e., $\alpha_{i} = \frac{1}{K}$ for all $i \in [K]$.
%For all $i$, define the function $f_{i}$ over $\X$ as follows:
%\begin{align*}
%	f_{i}(x_{j}) \eq
%	\begin{cases}
%		x_{j}^{2} & \text{if $j \ne i$,} \\
%		x_{i}^{2} - \eps^{2} & \text{if $j=i$.}
%	\end{cases}
%\end{align*}

\end{document}